%% file: main_arxiv.tex
\documentclass{article}

\usepackage{color,soul,fullpage,url}
\usepackage{amsmath,amssymb,amsthm}
\usepackage{bbm}
\usepackage{natbib}

\input{macros}

\begin{document}

\title{On Learnability under General Stochastic Processes}

\author{A. Philip Dawid and Ambuj Tewari\\
\texttt{apd@statslab.cam.ac.uk,tewaria@umich.edu}}


\maketitle

\begin{abstract}
Statistical learning theory under independent and identically distributed (iid) sampling and online learning theory for worst case individual sequences are two of the best developed branches of learning theory. Statistical learning under general non-iid stochastic processes is less mature. We provide two natural notions of learnability of a function class under a general stochastic process. We show that both notions are in fact equivalent to online learnability. Our results hold for both binary classification and regression.
\end{abstract}

\input{main_content}

\subsection*{Acknowledgments}
Thanks to the organizers and attendees of the Fifth Bayesian, Fiducial, and Frequentist Conference (BFF5)
held from May 6-9, 2018 in Ann Arbor, MI, USA. Their feedback on a preliminary version of this work was
very useful. We also thank the two HDSR (Harvard Data Science Review) reviewers who made numerous helpful suggestions to improve the paper.

\bibliographystyle{plainnat}
\bibliography{references.bib}

\appendix

\input{appendix_content}

\end{document}

%% file: macros.tex
\newcommand\X{\mathcal{X}}
\newcommand\Y{\mathcal{Y}}
\newcommand\Z{\mathcal{Z}}
\newcommand\F{\mathcal{F}}
\newcommand\x{\mathbf{x}}
\newcommand\y{\mathbf{y}}
\newcommand\z{\mathbf{z}}
\newcommand\s{\mathbf{s}}

\newcommand\ind[1]{\mathbbm{1}\left[#1\right]}
\newcommand\reals{\mathbb{R}}

\newcommand\fhat{\widehat{f}}
\newcommand\fstar{f^\star}

\ifdefined\E
\renewcommand\E[1]{\mathbb{E}\left[#1\right]}
\else
\newcommand\E[1]{\mathbb{E}\left[#1\right]}
\fi

\newcommand\Es[2]{\mathbb{E}_{#1}\left[#2\right]}
\DeclareMathOperator*{\argmin}{argmin}
\newcommand\prob[1]{\mathbb{P}\left(#1\right)}

\newcommand\pp{\mathbf{P}}
\newcommand\Pave{\bar{P}_n}

\newcommand\valiid{V^{\text{iid}}}
\newcommand\vcdim{\text{VCdim}}
\newcommand\fat{\text{fat}}
\newcommand\fatgamma{\fat_\gamma}
\newcommand\sfat{\text{sfat}}
\newcommand\sfatgamma{\sfat_\gamma}
\newcommand\rad{\mathfrak{R}}

\newcommand\valonline{V^{\text{online}}}
\newcommand\A{\mathcal{A}}
\newcommand\ldim{\text{Ldim}}
\newcommand\seqrad{\mathfrak{R}^\text{seq}}

\newcommand\pregret{\rho}
\newcommand\pregreto{\rho^{\text{online}}}
\newcommand\valgen{V^{\text{gen}}}
\newcommand\fhaterm{\fhat^{\text{ERM}}}

\newcommand\valpreq{V^{\text{preq}}}

\newcommand\bitv{\mathbf{b}}

\newcommand\tf{\tilde{f}}
\newcommand\tP{\tilde{P}}

\newcount\Comments  
\usepackage{color}
\definecolor{darkgreen}{rgb}{0,0.5,0}
\definecolor{purple}{rgb}{1,0,1}
\newcommand{\kibitz}[2]{\ifnum\Comments=1\textcolor{#1}{#2}\fi}

\newcommand\newedit[1]{{#1}} 

\newtheorem{theorem}{Theorem}
\newtheorem{lemma}[theorem]{Lemma}
\theoremstyle{definition}
\newtheorem{definition}{Definition}

%% file: main_content.tex
\section{Introduction}

One of the most beautiful and best developed branches of machine learning theory is classical statistical learning theory (see the article by \citet{von-luxburg2011statistical} for a non-technical overview and for more extensive references). However, it deals primarily with independent and identically distributed (iid) sampling of examples. There have been several attempts to deal with both dependence and non-stationarity: we discuss some of these extensions in Section~\ref{sec:related}. However in general the non-iid case is not as well developed as the classical iid case. 

Another well developed branch of learning theory that has its own share of elegant mathematical ideas is online learning theory (the book by~\citet{cesa-bianchi2006prediction} is an excellent if somewhat dated introduction). With roots in game theory and the area of information theory known as universal prediction of individual sequences, online learning theory, unlike statistical learning theory, does {not} use probabilistic foundations. It is therefore quite surprising that there are uncanny parallels between iid learning theory and online learning theory. The reader is invited to compare the statements of the fundamental theorems in these two areas (restated in this paper as Theorem~\ref{thm:SLT} and Theorem~\ref{thm:OLT}).

Our main goal in this paper is to study learnability of a function class in the statistical setting under extremely general assumptions that do not require independence or stationarity. We first summarize the key theorems of iid and online learning in Section~\ref{sec:statistical} and Section~\ref{sec:online}. Although this material is not new, we feel that the broader data science community might not be very familiar with results in online learning since it is a younger field compared to statistical learning theory. Also, presenting both iid learning and online learning results in a unified way allows us to draw parallels between the two theories and to motivate the need for novel theories that connect these two. 

We propose a definition of learnability under general stochastic processes in Section~\ref{sec:general}. We show that learning under this general definition is equivalent to online learnability (Theorem~\ref{thm:general}). We give a prequential version of our main definition in Section~\ref{sec:prequential}. In the prequential version, as in online learning, the function output by the learning algorithm at any given time cannot peek into the future. We show that learnability under the prequential version of our general learning setting is also equivalent to online learnability (Theorem~\ref{thm:preq}). We focus on the problem of binary classification for simplicity. But we also provide extensions of our equivalence results to the problem of real valued prediction (i.e., regression) in Section~\ref{sec:regression} (see Theorem~\ref{thm:reg_general} and Theorem~\ref{thm:reg_preq}).

\subsection{Related Work}
\label{sec:related}

The iid assumption of statistical learning theory has been relaxed and replaced with various types of {\em mixing} assumptions, especially $\beta$-mixing~\citep{vidyasagar2002theory,mohri2009rademacher}. However, in this line of investigation, the stationary assumption is kept and the theory resembles the iid theory to a large extent since mixing implies approximate independence of random variables that are sufficiently separated in time. Mixing assumptions can be shown to hold for some interesting classes of processes, including some Markov and hidden Markov processes. Markov sampling has also been considered on its own as a generalization of iid sampling~\citep{aldous1995markovian,gamarnik2003extension,smale2009online}.

There has been work on performance guarantees of specific algorithms like boosting \citep{lozano2006convergence} and SVMs~\citep{steinwart2009consistency,steinwart2009learning} under non-iid assumptions. However, our focus here is not on any specific learning methodology. We would like to point out that, while we focus on {\em learnability} of functions in a fixed class, the question of {\em universal consistency} has also been studied in the context of general stochastic processes~\citep{nobel1999limits,hanneke2017learning}.

There are a handful of papers that focus, as we do, on conditional risk given a sequence of observation drawn from a general non-iid stochastic processes \citep{pestov2010predictive,shalizi2013predictive,zimin2017learning}. These papers focus on process decompositions: expressing a complex stochastic process as a mixture of simpler stochastic processes. For example, de Finetti's theorem shows that exchangeable distributions are mixtures of iid distributions. The basic idea is to output a function with small expected loss one step beyond the observed sample where the expectation is also conditioned on the observed sample. While closely related, our performance measures are cumulative in nature and are inspired more by regret analysis in online learning than PAC bounds in computational learning theory.

The use of tools from online learning theory (e.g., sequential Rademacher complexity) for developing learning theory for dependent, non-stationary process was pioneered by Kuznetsov and Mohri~\citep{kuznetsov2015learning,kuznetsov2017generalization}. However, their focus is on time series forecasting applications and therefore their performance measures always involve the expected loss of the function chosen by the learning algorithm some steps into the future (i.e., the part not seen by the learning algorithm) of the process. In contrast, our definition uses conditional distributions of the stochastic process {\em on the realized path} to define our performance measure. We also point out that there are earlier papers that apply learning theory tools to understand time series prediction~\citep{modha1998memory,meir2000nonparametric,alquier2012model}. Some very recent work has also begun to extend some of the work on time series to processes with spatial structure and dependence such as those occurring on a network~\citep{dagan2019learning}.

A direct inspiration for this paper is the work of \citet{skouras2000consistency} on estimation in semi-parametric statistical models under misspecification. They highlighted that, under misspecification, M-estimators, including maximum likelihood estimators, may not converge to a deterministic limit even asymptotically. Instead, the limit can be stochastic. This is because, under misspecification, the ``best'' model can depend on the observed sequence of data. They gave examples showing that this can happen for non-ergodic processes or processes with long range dependencies that do not decay fast enough. Our work can be seen as a direct extension of their ideas to the learning theory setting where the focus is not on parameter estimation but on loss minimization over potentially massive function spaces.

\section{Preliminaries}

We consider a supervised learning setting where we want to learn a mapping from an input space $\X$ to an output space $\Y$. Two output spaces of interest to us in this paper are $\Y = \{-1,+1\}$ (binary classification) and $\Y = [-1,+1]$ (regression). Instead of talking about the difficulty of learning individual functions, we will define learnability for a {\em class} of functions that we will denote by $\F \subseteq \Y^{\X}$. Let $\Z = \X \times \Y$ and let $\ell: \Z \times \F \to \reals _+$ be a loss function mapping an input-output pair $(x,y)$ and a function $f$ to a non-negative loss. The set $\{1,\ldots,n\}$ will be denoted by $[n]$ and we use $\ind{C}$ to denote an indicator function that is $1$ if the condition $C$ is true and $0$ otherwise. Two important loss functions are the $0$-$1$ loss $\ell((x,y),f) = \ind{y \neq f(x)}$ (in binary classification) and the absolute loss $\ell((x,y),f) = |y - f(x)|$ (in regression).

We often denote an input-output pair $(x,y)$ by $z$. When the input-output pair is random, we will denote it by $Z = (X,Y)$, perhaps with additional time indices such as $Z_t = (X_t,Y_t)$. We will use the abbreviation $Z_{1:t}$ to denote the sequence $Z_1,\ldots,Z_t$. A learning rule $\fhat_n$ is a map from $\Z^n$ to $\F$. We will abuse notation a bit and refer to the learning rule and the function output by the learning rule both by $\fhat_n$. An important learning rule is empirical risk minimization (ERM): given a sequence $z_{1:t}$ of input-output pairs, it outputs the function,
\begin{equation}\label{eq:erm}
\fhaterm_n = \argmin_{f \in \F} \frac{1}{n} \sum_{t=1}^n \ell(z_t,f) .
\end{equation}
Note that, for infinite function classes, the minimum may not be achieved. In that case, one can work with functions achieving empirical risks that are arbitrarily close to the infimum of the empirical risk over the class $\F$.

Given a distribution $P$ on $\Z$, the loss function can be extended as follows:
\[
\ell(P, f) = \Es{z\sim P}{\ell(z,f)} .
\]
The function minimizing the expectation above is
\[
\fstar_P = \argmin_{f \in \F} \ell(P, f).
\]
The $P$-regret of a function $f \in \F$ is defined as
\begin{align*}
\pregret(P, f) &= \ell(P, f) - \inf_{f' \in \F} \ell(P, f') \\
&= \ell(P, f) - \ell(P, \fstar_P).
\end{align*}
Note that the $P$-regret depends on the class $\F$ but we hide this dependence when the function class is clear from the context.

\section{Learnability in the IID Setting}
\label{sec:statistical}

In this section we review some basic results of statistical learning theory under iid sampling. For more details the reader can consult standard texts in this area~\citep{anthony1999neural,vidyasagar2002theory,shalev2014understanding}.
In the standard formulation of {statistical} learning theory, we draw a sequence $Z_{1:n}$ of iid examples from a distribution $P$. That is, the joint distribution of $Z_{1:n}$ is a product distribution $\pp = P \otimes P \otimes \ldots \otimes P$. We adopt the minimax framework to define learnability of a class $\F$ of functions with respect to a loss function $\ell$. Define the worst case performance of a learning rule $\fhat_n$ by
\[
\valiid_n(\fhat_n, \F) = \sup_{P} \E{ \pregret(P, \fhat_n) } 
\]
and the minimax value by
\[
\valiid_n(\F) = \inf_{\fhat_n} \valiid_n(\fhat_n, \F) .
\]
For the sake of conciseness, the notation above hides the fact that $\fhat_n$ depends on the sequence $Z_{1:n}$. The expectation above is taken over the randomness in these samples.

\begin{definition}
We say that $\F$ is learnable in the iid learning setting if
\[
\limsup_{n \to \infty} \valiid_n(\F) = 0 .
\]
Furthermore, we say that $\F$ is learnable via a sequence $\fhat_n$ of learning rules if
\[
\limsup_{n \to \infty} \valiid_n(\fhat_n, \F) = 0 .
\]
\end{definition}

One of the major achievements of statistical learning theory was the determination of necessary and sufficient conditions for learnability of a class $\F$. Learnability in both binary classification with $0$-$1$ loss and regression with absolute loss is known to be equivalent to a probabilistic condition, namely the uniform law of large numbers (ULLN) for the class $\F$:
\begin{equation}\label{eq:ULLN}
\limsup_{n \to \infty} \sup_{P} \E{ \sup_{f \in \F} \left| \frac{1}{n} \sum_{t=1}^n f(X_t) - Pf \right|} = 0 .
\end{equation}

Here $X_{1:n}$ are drawn iid from $P$ and $Pf = \Es{X \sim P}{f(X)}$.
Whether or not ULLN holds for a class $\F$ depends on the finiteness of different combinatorial parameters, depending on whether we are in the binary classification or regression setting. We will discuss the binary classification case here, leaving the regression case to Section~\ref{sec:regression}.

The VC dimension of $\F$, denoted by $\vcdim(\F)$, is the length $n$ of the longest sequence $x_{1:n}$ shattered by $\F$. We say that a sequence $x_{1:n}$ is shattered by $\F$ if
\[
\forall \epsilon_{1:n} \in \{\pm1\}^n, \exists f \in \F, \text{ s.t. } \forall t \in [n], f(x_t) = \epsilon_t .
\]
Finally, we recall the definition of the (expected) Rademacher complexity of a function class with respect to a distribution $P$:
\[
\rad_n(P, \F) = \E{ \sup_{f \in \F} \frac{1}{n} \sum_{t=1}^n \epsilon_t f(X_t) }
\]
Note that the expectation above is with respect to both $X_{1:n}$ and $\epsilon_{1:t}$. The former are drawn iid from $P$ whereas the latter are iid $\{\pm 1\}$-valued Rademacher (also called symmetric Bernoulli) random variables. The worst case, over $P$, Rademacher complexity is denoted by
\[
\rad_n(\F) = \sup_{P} \rad_n(P, \F) .
\]

\begin{theorem}\label{thm:SLT}
Consider binary classification with $0$-$1$ loss in the iid setting. Then, the following are equivalent:
\begin{enumerate}
    \item $\F$ is learnable.
    \item $\F$ is learnable via ERM.
    \item The ULLN condition~\eqref{eq:ULLN} holds for $\F$.
    \item $\vcdim(\F) < \infty$. \label{vccond}
    \item $\limsup_{n \to \infty} \rad_n(\F) = 0$.
\end{enumerate}
\end{theorem}

A similar result holds for regression with absolute loss with the VC dimension condition (i.e., condition number \ref{vccond} above) replaced with a similar one involving its scale-sensitive counterpart, called the fat shattering dimension (see Section~\ref{sec:reg_statistical} for details).

\section{Learnability in the Online Setting}
\label{sec:online}

A second learning setting with a well-developed theory is the {\em online} learning setting, where no probabilistic assumptions are placed on the data-generating process. Compared to statistical learning theory under iid sampling, online learning theory is a younger field. The main combinatorial parameter in this area, the Littlestone dimension, was defined by~\citet{littlestone1988learning}. It was given the name ``Littlestone dimension'' by \citet{ben-david2009agnostic}, where it was also shown that it fully characterizes learnability in the binary classification setting. Scale-sensitive analogues of Littlestone dimension for regression problems and the sequential version of Rademacher complexity were studied in~\citet{rakhlin2015online,rakhlin2015sequential}.

The online learning setting takes an individual sequence approach, where results are sought that hold for every possible sequence $z_{1:n} \in \Z^n$ that might be encountered by the learning rule. 

We consider a sequence $\fhat_{0:n-1}$ of learning rules, where $\fhat_t$ takes in as input the sequence $z_{1:t}$ and outputs a (possibly random) function in $\F$.
Define the expected (normalized) regret of $\fhat_{0:n-1}$ on sequence $z_{1:n}$:
\[
\pregreto(\fhat_{0:n-1},z_{1:n}) = 
\E{
\frac{1}{n}
\left( 
\sum_{t=1}^n \ell(z_t, \fhat_{t-1}) - \inf_{f \in \F} \sum_{t=1}^n \ell(z_t, f) 
\right)
}.
\]
This is similar in flavor to, but distinct from, the regret function $\pregret$ used in the iid setting.
 It obeys the {\em prequential principle}~\citep{dawid1984present}: performance of $\fhat_{t-1}$, which is learned using $z_{1:t-1}$, is judged using loss evaluated on $z_t$ with no overlap between data used for learning and for performance evaluation. The expectation is needed because the learning rules $\fhat_{0:n-1}$ may use internal randomization to achieve robustness to adversarial data.  The regret nomenclature comes from the fact that $\fhat_{0:n-1}$ cannot peek into the future to lower its loss but its cumulative performance is compared with lowest possible loss, in hindsight, over the entire sequence $z_{1:n}$. However, the comparator term has its own restriction: it uses the best fixed function $f$ in hindsight, as opposed to the best sequence of functions.

The object of interest is now the following minimax value:
\[
\valonline_n(\F) = \inf_
{\fhat_{0:n-1}} \valonline(\fhat_{0:n-1}, \F),
\]
where 
\[
\valonline(\fhat_{0:n-1}, \F) = \sup_{z_{1:n} \in \Z^n} \pregreto(\fhat_{0:n-1},z_{1:n})
\]
is the worst-case performance of the sequence $\fhat_{0:n-1}$ of learning rules, with $\fhat_t$ taking in as input the sequence $z_{1:t}$ and outputting a function in $\F$.
The infimum is then taken over all such learning rule sequences.

\begin{definition}
We say that $\F$ is learnable in the online learning setting if
\[
\limsup_{n \to \infty} \valonline_n(\F) = 0 .
\]
\end{definition}

As in statistical learning, we have necessary and sufficient conditions for learnability that almost mirror those in Theorem~\ref{thm:SLT}. The ULLN condition gets replaced by the Uniform Martingale Law of Large Numbers (UMLLN). We say that UMLLN holds for $\F$ if
\begin{equation}\label{eq:UMLLN}
\limsup_{n \to \infty} \sup_{\pp, \A} \E{ \sup_{f \in \F} \left| \frac{1}{n} \sum_{t=1}^n \left( f(X_t) - \E{f(X_t)|\A_{t-1}} \right) \right|} = 0 .
\end{equation}
The crucial difference between the UMLLN condition and the ULLN condition is that here the supremum is taken over {\em all} joint distributions $\pp$ of $X_{1:n}$. In particular $X_{1:n}$ need not be iid. Also, to obtain a martingale structure, we use an arbitrary filtration $\A = (\A_t)_{t=0}^{n-1}$ such  that $X_t$ is $\A_t$-measurable.  It is easy to see that UMLLN is a stronger condition than ULLN: simply restrict $\pp$ to be a product distribution and let $\A$ be the natural filtration of $X_t$. Then the UMLLN condition reduces to the ULLN condition.

The VC dimension of $\F$ is replaced by another combinatorial parameter, called the Littlestone dimension of $\F$, denoted by $\ldim(\F)$. Before we present the definition of Littestone dimension, we need some notation to handle complete binary trees labeled with examples drawn from the input space $\X$. We think of a complete binary tree $\x$ of depth $n$ as defining a sequence $\x_t, 1 \le t \le n$, of maps. The map $\x_t$ gives us the examples sitting at level $t$ of the tree. For example, $\x_1$ is the root, $\x_2(-1)$ is the left child of the root, $\x_2(+1)$ is the right child of the root, and so on. In general $\x_t(\epsilon_{1:t-1})$ is the node at level $t$ that we reach by following the path given by the sign sequence $\epsilon_{1:t-1} \in \{\pm1\}^{t-1}$, where $-1$ means ``go left'' and $+1$ means ``go right''. The Littlestone dimension of $\F$ is the depth $n$ of the largest complete binary tree $\x$ shattered by $\F$. We say that a complete binary tree $\x$ is shattered by $\F$ if
\[
\forall \epsilon_{1:n} \in \{\pm1\}^n, \exists f \in \F, \text{ s.t. } \forall t \in [n], f(\x_t(\epsilon_{1:t-1})) = \epsilon_t .
\]

Finally, Rademacher complexity gets replaced with its sequential analogue, called the sequential Rademacher complexity. We first define the sequential Rademacher complexity of $\F$ given a tree $\x$ of depth $n$ as:
\[
\seqrad_n(\x, \F) = \E{ \sup_{f \in \F} \frac{1}{n} \sum_{t=1}^n \epsilon_t f(\x_t(\epsilon_{1:t-1})) } .
\]
Note that the expectation above is only with respect to the Rademacher random variables $\epsilon_{1:t}$ as $\x$ is a fixed tree. Taking the worst case over all complete binary trees $\x$ of depth $n$ gives us the sequential Rademacher complexity of $\F$:
\[
\seqrad_n(\F) = \sup_{\x} \seqrad_n(\x, \F) .
\]

\begin{theorem}\label{thm:OLT}
Consider binary classification with $0$-$1$ loss in the online (individual sequence) setting. Then, the following are equivalent:
\begin{enumerate}
    \item $\F$ is learnable.
    \item The UMLLN condition~\eqref{eq:UMLLN} holds for $\F$.
    \item $\ldim(\F) < \infty$. \label{lcond}
    \item $\limsup_{n \to \infty} \seqrad_n(\F) = 0$.
\end{enumerate}
\end{theorem}

As in the iid setting, a similar result holds for online regression with absolute loss, with the Littlestone dimension condition (i.e., condition number \ref{lcond} above) replaced by a similar one involving its scale-sensitive counterpart, called the sequential fat shattering dimension (see Section~\ref{sec:reg_online} for details).

It is well known that online learnability is harder than iid learnability. That is, $\vcdim(\F) \le \ldim(\F)$ for any $\F$, and the gap in this inequality can be arbitrarily large. For example, the set of threshold functions on $\reals$:
\begin{equation}
\label{eq:thresh}
\F_{\text{threshold}}
= \{
x \mapsto \ind{x > \theta}
\::\:
\theta \in \reals
\}
\end{equation}
has $\vcdim(\F_{\text{threshold}}) = 1$
but $\ldim(\F_{\text{threshold}}) = \infty$.

A conspicuous difference between Theorem~\ref{thm:SLT} and Theorem~\ref{thm:OLT} is the absence of the condition involving ERM. Indeed, ERM is not necessarily a good learning rule in the online setting: there exist classes learnable in the online setting that are not learnable via ERM. Unfortunately, the learning rules that learn a class $\F$ in the online setting are quite complex \citep{ben-david2009agnostic}. It is not known if there exists a rule as simple as ERM that will learn a class $\F$ whenever $\F$ is online learnable. In any case, ERM does {not} play as central a role in online learning as it does in learning in the iid setting.

\section{Learnability under General Stochastic Processes}
\label{sec:general}

In this section we move beyond the iid setting to cover {\em all} distributions, not just product distributions. 
For a general stochastic process $\pp$, we still have an analogue of $P$ at time $t$, namely \[
P_t(\cdot; z_{1:t-1}) = \pp(\cdot|Z_{1:t-1}=z_{1:t-1}) .
\]
This is the conditional distribution of $Z_t$ given $Z_{1:t-1}$. Just like $P$, this is unknown to the learning rule. However, unlike $P$ in the iid case, $P_t$ is {\em data-dependent}. Therefore the $P_t$-regret of a function $\pregret(P_t, f)$ is data-dependent. We will often hide the dependence of $P_t$ on past data $Z_{1:t-1}$. We can use the average of the $P_t$-regrets,
\[
R_n(Z_{1:n}, f) = \frac{1}{n} \sum_{t=1}^n \pregret(P_t, f)
\]
as a performance measure. Note that the minimizer of this performance measure is data-dependent, unlike in the iid case. \newedit{As in the iid setting a learning rule $\fhat_n$ is a map from $\Z^n$ to $\F$. To reduce clutter in our notation, we will continue to hide the dependence of $\fhat_n$ on the realized sample $Z_{1:n}$.} The value of a learning rule $\fhat_n$ is now defined as
\begin{align*}
\valgen_n(\fhat_n, \F)
&= \sup_{\pp} \E{
R_n(Z_{1:n}, \fhat_n) - \inf_{f \in \F} R_n(Z_{1:n}, f)
} \\
&= \sup_{\pp} \E{
 \frac{1}{n} \sum_{t=1}^n \ell(P_t, \fhat_n)
                        - \inf_{f \in \F} \frac{1}{n} \sum_{t=1}^n \ell(P_t, f)
} ,
\end{align*}
where the supremum is now taken over all joint distributions $\pp$ over $Z_{1:n}$.
This leads to consideration of the following minimax value to define learnability:
\[
\valgen_n(\F) = \inf_{\fhat_n} \valgen_n(\fhat_n, \F) .
\]

\begin{definition}
We say that $\F$ is \newedit{process learnable} if
\[
\limsup_{n \to \infty} \valgen_n(\F) = 0 .
\]
Furthermore, we say that $\F$ is \newedit{process learnable} via a sequence $\fhat_n$ of learning rules if
\[
\limsup_{n \to \infty} \valgen_n(\fhat_n, \F) = 0 .
\]
\end{definition}

Note that in the iid case, when $\pp$ is a product distribution with marginal $P$, we have $P_t = P$ for all $t$ and therefore, for any $f$,
\[
\frac{1}{n} \sum_{t=1}^n \ell(P_t, f) = \ell(P, f) .
\]
We have the following result as an immediate consequence.

\begin{lemma}
\label{lem:genvsstat}
Fix any loss function $\ell$ and function class $\F$. For any learning rule $\fhat_n$, $\valgen(\fhat_n, \F) \ge \valiid(\fhat_n, \F)$. This also means that $\valgen_n(\F) \ge \valiid_n(\F)$.
\end{lemma}

The result above is not surprising: \newedit{process learnability has to be harder than iid learnability}. However, somewhat surprisingly, we can show that \newedit{process learnability} is at least as hard as online learnability.

\begin{theorem}
\label{thm:ldimlowerbound}
Consider binary classification with 0-1 loss in the general stochastic process setting. Suppose the class $\F$ is not online learnable, i.e., $\ldim(\F) = \infty$. Then for any $n \ge 1$,
$\valgen_n(\F) \ge 1/8$. Therefore, the class $\F$ is not \newedit{process learnable}.
\end{theorem}

To complement the lower bound above, we will now give a performance guarantee for ERM in the general stochastic process setting. Given a loss $\ell$ and function class $\F$, define the loss class $\ell \circ \F$ as
\[
\ell \circ \F =
\{
z \mapsto \ell(z,f) \::\:
f \in \F 
\} .
\]
We define the sequential Rademacher complexity of a loss class $\ell \circ \F$ as
\begin{align*}
\seqrad_n(\z, \ell \circ \F) &= \E{ \sup_{f \in \F} \frac{1}{n} \sum_{t=1}^n \epsilon_t \ell(\z_t(\epsilon_{1:t-1}), f) } ,\\
\seqrad_n(\ell \circ \F) &= \sup_{\z} \seqrad_n(\z, \ell \circ \F) .
\end{align*}
Note that the supremum here is over $\Z$-valued trees that are labeled with input-output pairs.
It is easy for us to connect the complexity to the loss class to the complexity of the underlying function class for a simple loss function like the $0$-$1$ loss  (see Appendix~\ref{sec:losstofunction} for details.)

\begin{theorem}
\label{thm:gentorad}
Fix any loss function $\ell$ and function class $\F$. Let $\fhaterm$ denote the ERM learning rule defined in~\eqref{eq:erm}. Then we have
\[
\valgen_n(\F) \le \valgen_n(\fhaterm_n, \F) \le 4 \seqrad_n(\ell \circ \F).
\]
\end{theorem}
\begin{proof}
The first inequality is true by definition of $\valgen_n(\F)$. So we just have to prove the second one.

Note, by definition of $\fhaterm$,
\[
\frac{1}{n} \sum_{t=1}^n \ell(Z_t, \fhaterm_n) =
\inf_{f \in \F} \frac{1}{n} \sum_{t=1}^n \ell(Z_t, f) .
\]
Therefore, we have
\begin{align}
\notag
&\quad R_n(Z_{1:n}, \fhaterm_n) - \inf_{f \in \F} R_n(Z_{1:n}, f) \\
\notag
&= \frac{1}{n} \sum_{t=1}^n \ell(P_t, \fhaterm_n)
    - \inf_{f \in \F} \frac{1}{n} \sum_{t=1}^n \ell(P_t, f) \\
\notag
&=  \frac{1}{n} \sum_{t=1}^n \ell(P_t, \fhaterm_n)
    -  \frac{1}{n} \sum_{t=1}^n \ell(Z_t, \fhaterm_n) + \inf_{f \in \F} \frac{1}{n} \sum_{t=1}^n \ell(Z_t, f)
    - \inf_{f \in \F} \frac{1}{n} \sum_{t=1}^n \ell(P_t, f) \\
&\le \sup_{f \in \F} \frac{1}{n} \left( \sum_{t=1}^n \ell(P_t, f)
    - \ell(Z_t, f) \right) + \sup_{f \in \F} \frac{1}{n} \left( \sum_{t=1}^n \ell(Z_t, f)
    - \ell(P_t, f) \right) . \label{eq:twosups}
\end{align}
The justification for the last inequality is as follows. First, we know that $\fhaterm_n \in \F$. Second, when $\inf_{f \in \F} \frac{1}{n} \sum_{t=1}^n \ell(P_t, f) $ is achieved, at $\fstar$ say, we have,
\begin{align*}
&\inf_{f \in \F} \frac{1}{n} \sum_{t=1}^n \ell(Z_t, f)  - \inf_{f \in \F} \frac{1}{n} \sum_{t=1}^n \ell(P_t, f)\\
&\leq \frac{1}{n} \sum_{t=1}^n \ell(Z_t, \fstar)  - \frac{1}{n} \sum_{t=1}^n \ell(P_t, \fstar) \\
&\leq \sup_f \frac{1}{n} \left(\sum_{t=1}^n \ell(Z_t, f)  - \frac{1}{n} \sum_{t=1}^n \ell(Z_t, f)\right).
\end{align*}
Taking expectations on both sides of~\eqref{eq:twosups} gives us
\begin{align*}
& \quad \E{ R_n(Z_{1:n}, \fhaterm_n) - \inf_{f \in \F} R_n(Z_{1:n}, f) } \\
&\le \E{
\sup_{f \in \F} \frac{1}{n} \left( \sum_{t=1}^n \ell(P_t, f) - \ell(Z_t, f) \right)
} \\
&\quad + \E{
\sup_{f \in \F} \frac{1}{n} \left( \sum_{t=1}^n \ell(Z_t, f)
    - \ell(P_t, f) \right)
} \\
&\le 4 \seqrad_n(\ell \circ \F) .
\end{align*}
Note that the last inequality follows from Theorem 2 of \citet{rakhlin2015sequential}. Since the last quantity above does not depend on $\pp$, we can take supremum over $\pp$ on both sides to finish the proof.
\end{proof}

We now have everything in place to be able to show the equivalence of \newedit{process learnability} and online learnability. \newedit{A similar result can also be shown in the regression case (see Section~\ref{sec:reg_general}).}

\begin{theorem}\label{thm:general}
Consider binary classification with $0$-$1$ loss. Then all of the equivalent conditions in Theorem~\ref{thm:OLT} are also equivalent to:
\begin{itemize}
    \item $\F$ is \newedit{process learnable.}
\end{itemize}
\end{theorem}
\begin{proof}
Theorem~\ref{thm:ldimlowerbound} established that learnability in the general stochastic process setting implies online learnability. For the other direction, note that according to Theorem~\ref{thm:gentorad} we have
\[
\valgen_n(\F) \le 4 \, \seqrad_n(\ell \circ \F) 
\le
2 \, \seqrad_n(\F) \ ,
\]
where the second inequality follows from Theorem~\ref{thm:radzeroone} in Appendix~\ref{sec:losstofunction}.
Taking $\lim \sup$ of both sides as $n$ tends to infinity shows that online learnability implies \newedit{process learnability}.
\end{proof}

Although online learnability turns out to be equivalent to \newedit{process learnability}, there is an important difference between the two settings which has to do with the importance of ERM. In the former ERM is not a good learning rule whereas in the latter a learnable class is learnable via ERM. Therefore ERM continues to play a special role in the general stochastic process setting just like the iid setting.

\newedit{Also note that Theorem~\ref{thm:general} is stated in terms of learnability which is an asymptotic concept. However, the proof clearly shows that the rate of convergence is determined by the sequential Rademacher complexity of $\F$ which scales as $O\left(\sqrt{{\ldim(\F)}/{n}}\right)$~\citep{alon2021adversarial}}

\subsection{Examples}

We end this section with some examples showing that our definition of \newedit{process learnability} is natural, interesting and worth studying.

\textbf{IID Sampling.} Let us note once again that if $\pp = P \otimes P \otimes \ldots \otimes P$ is a product measure then $\ell(P_t, f)$ is just $\ell(P, f)$ and therefore not random. In this special but important case, our definition of learnability reduces to the standard definition of learnability under iid sampling.

\textbf{Asymptotically Stationary Process.} Suppose that $\pp$ is not a product measure but the process is {\em asymptotically stationary} in the sense that the random probability measure $\Pave = \frac{1}{n} \sum_{t=1}^n P_t$ converges to some fixed deterministic $P^\star$ in total variation $\|\cdot\|_{TV}$ as $n \to \infty$. For a class $\F$ that is learnable in the general stochastic process setting and for loss function bounded by $1$, we have
\begin{align*}
&\quad
\E{ \ell(P^\star, \fhaterm_n)} - \inf_{f \in \F} \ell(P^\star, f) \\
&=
\E{ \ell(P^\star, \fhaterm_n) - \ell(P^\star, f_{P^\star})} \\
&\le
2 \, \E{ \sup_{f \in \F} |
\ell(P^\star, f) - \ell(\Pave, f)
| } \\
&\quad + \E{
\ell(\Pave, \fhaterm_n) - \ell(\Pave, f_{P^\star}) } \\
&\le
2 \, \E{ \|P^\star - \Pave\|_{TV} } \\
&\quad + \E{
\ell(\Pave, \fhaterm_n) - \inf_{f \in \F} \ell(\Pave, f) 
} .
\end{align*}
By the stationarity assumption, the first term on the right in the last inequality goes to zero. \newedit{Moreover, the rate of convergence can often be characterized in terms of the mixing coefficients of the stochastic process~\citep{vidyasagar2002theory}.} By learnability of $\F$ via ERM in the general stochastic process setting, the last term goes to zero. Note that $\ell(P, f)$ is linear in $P$ and therefore $\ell(\Pave, f) = \frac{1}{n} \sum_{t=1}^n \ell(P_t, f)$. So, under stationarity, our learnability condition implies that ERM does well when its performance is measured under the (asymptotic) stationary distribution $P^\star$.

\color{black}
\textbf{Mixture of IID.} Consider a simple mixture of product distributions
\[
\pp = \lambda P \otimes P \otimes \ldots \otimes P +
(1-\lambda) Q \otimes Q \otimes \ldots \otimes Q
\]
where, for simplicity, assume that $P$ and $Q$ have disjoint supports. Then with probability $\lambda$ we have $\forall t>1$ $P_t = P$, and with probability $1-\lambda$ we have $\forall t>1$ $P_t = Q$. Therefore, the minimizer of
\begin{align}
\label{eq:perfmixture}
\frac{1}{n} \sum_{t=1}^n \ell(P_t, f)
\end{align}
is $\fstar_{P}$ with probability $\lambda$ and $\fstar_{Q}$ with probability $1-\lambda$ (assuming, again for simplicity, that the minimizers $\fstar_{P}, \fstar_{Q}$ are unique). Here, unlike the iid and stationary examples, the ``best'' function, even with infinite data, is not deterministic but is random depending on which mixture component was selected. Still, if learnability in our general sense holds, then ERM will do well according to the performance measure~\eqref{eq:perfmixture}. Note that this example can be easily generalized to a mixture of more than two iid processes. It can also be generalized, with additional technical conditions, to the case when $P \neq Q$. The main difference from the disjoint support case that we consider here will be that with probability $\lambda$, $P_t$ will converge to $P$ (in a suitable sense) and to $Q$ otherwise. Similarly, the minimizer of~\eqref{eq:perfmixture} would not equal $\fstar_P$ or $\fstar_Q$ but it would converge (again, in some appropriate sense determined by technical conditions) to one of them with probability $\lambda$ and $1-\lambda$ respectively.
\color{black}

\textbf{Random Level.} Fix the squared loss $\ell(z,f) = (y-f(x))^2$ and consider a class $\F$ that is iid learnable and closed under translations by a constant, i.e., if $f \in \F$ then $f + c \in \F$ for any constant $c \in \reals$. Let $X_{1:n}$ be iid drawn from some distribution $P_X$ on $\X \subseteq \reals^d$ that has a density with respect to Lebesgue measure on $\reals^d$. Let $Y_t = f^\star(X_t) + \xi_t + \xi_0$ for some $f^\star \in \F$ and $1 \le t \le n$ where $(\xi_t)_{t=0}^n$ are iid standard normal. Note that the process $Z_t = (X_t,Y_t), 1 \le t \le n$ is {\em not} iid. It is not even mixing in any sense due to long range dependence in $Y_t$ caused by $\xi_0$.
Now ERM over $\F$ is given by:
\begin{align*}
\fhaterm_n(Z_{1:n})
&= \argmin_{f \in \F} \frac{1}{n}
\sum_{t=1}^n (f^\star(X_t) + \xi_0 + \xi_t - f(X_t))^2 \\
&= \argmin_{f \in \F} \frac{1}{n}
\sum_{t=1}^n (f^\star(X_t)  + \xi_t - (f(X_t)-\xi_0))^2 \\
&= \xi_0 + \argmin_{g \in \F-\xi_0} \frac{1}{n}
\sum_{t=1}^n (f^\star(X_t) + \xi_t - g(X_t))^2 \\
&= \xi_0 + \argmin_{g \in \F} \frac{1}{n}
\sum_{t=1}^n (f^\star(X_t) + \xi_t - g(X_t))^2 ,
\end{align*}
where the last equality holds because $\F - \xi_0 = \F$ and we have assumed that all empirical minimizers are unique with probability $1$. Thus, we have shown that
\[
\fhaterm_n(Z_{1:n}) = \fhaterm_n((X_t,f^\star(X_t) + \xi_t)_{t=1}^n) + \xi_0 .
\]
Since $\F$ is iid learnable, $\fhaterm_n((X_t,f^\star(X_t) + \xi_t)_{t=1}^n)$ converges (in $L_2(P_X)$ sense) to the function $f^\star$ which means the ERM on $Z_{1:n}$ converges to the {\em random} function $f^\star + \xi_0$. 

Next we compute $\ell(P_t,f)$, as follows.
Let $P'_t$ be the conditional distribution of $Z_t$ given $X_{1:t-1}$ and $\xi_{0:t-1}$. Then we have
\begin{align*}
\ell(P'_t, f) &=
\E{ (Y_t - f(X_t))^2 | X_{1:t-1}, \xi_{0:t-1} } \\
&= \E{
(f^\star(X_t) + \xi_t + \xi_0 - f(X_t))^2 |
 X_{1:t-1}, \xi_{0:t-1}
} \\
& = \E{
(f^\star(X_t) + \xi_t + \xi_0 - f(X_t))^2 |
 \xi_{0}
} \\
&= 1 + \| f^\star - f + \xi_0  \|^2_{L_2(P_X)}
\end{align*}
(with $\xi_0$ regarded as fixed).
Then $\ell(P_t, f) = \E{\ell(P'_t) | Z_{1:t-1}} = 1 + \E{\| f^\star  - f + \xi_0 \|^2_{L_2(P_X)}| Z_{1:t-1}}$, where now $\xi_0$ (only) is regarded as random.  It is easy to show that the distribution of $\xi_0$, given $Z_{1:t-1}$, is normal with variance $1/t$ and mean 
$$U_{t-1} = \frac{\sum_{i=1}^{t-1}(Y_i-f^*(X_i))}{t}.$$
Consequently 
$$\ell(P_t,f)= 
1 + \frac 1 t + \left\| f^\star  - f + U_{t-1}\right\|^2_{L_2(P_X)}.$$
In particular, 
$$\inf_{f\in \F}\frac 1 n \sum_{t=1}^n \ell(P_t,f) \geq 1.$$

Now, with $\fhaterm_n = \fhaterm(Z_{1:n})$, consider $\ell(P_t,\fhaterm_n)$.  We have shown $f^*-\fhaterm_n \rightarrow -\xi_0$ in mean square.  Also, 
$$U_{t-1} = \xi_0 + \frac 1 t \left(\sum_{i=1}^{t-1}{\xi_i}-\xi_0\right)\rightarrow \xi_0$$ 
in mean square.  
So 
$\frac 1 n \sum_{t=1}^n \ell(P_t,\fhaterm_n) \rightarrow 1$, the smallest possible value.
That is, asymptotically the minimiser of $\frac{1}{T} \sum_{t=1}^T \ell(P_t, f)$
over $\F$ is $\fhaterm_T$ (which converges, not to $\fstar$, but to the random function $\fstar + \xi_0$).

\section{A Prequential Definition of Learnability}
\label{sec:prequential}

The previous section generalized the statistical setting to include non-product distributions and extended the definition of learnability to a more general setting. In this section we will generalize the online learnability definition to obtain a prequential version of learnability, while still keeping the level of generality of the previous section. As in the online setting, consider a sequence of learning rules $\fhat_{0:n-1}$, where $\fhat_t$ is a function only of $Z_{1:t-1}$, i.e. it cannot peek ahead to access $Z_{t:n}$. Unlike the online learning setting, $Z_{1:t}$ is a random sequence drawn from some general distribution $\pp$ over $\Z^n$. Now, define the minimax value
\[
\valpreq_n(\F) = \inf_{\fhat_{0:n-1}} \valpreq_n(\fhat_{0:n-1}, \F) ,
\]
where
\[
\valpreq_n(\fhat_{0:n-1}, \F) = \sup_{\pp} 
\E{
\frac{1}{n} \sum_{t=1}^n \ell(P_t, \fhat_{t-1}) - \inf_{f \in \F} \frac{1}{n} \sum_{t=1}^n \ell(P_t, f) 
}.
\]
Note that the expectation above is with respect to both $\pp$ and any internal randomness used by the rules $\fhat_{0:n-1}$. As before, the definition of the minimax value leads to the definition of learnability.
\begin{definition}
We say that $\F$ is \newedit{prequentially learnable} if
\[
\limsup_{n \to \infty} \valpreq_n(\F) = 0 .
\]
\end{definition}

The definition of $\valpreq_n(\F)$ can be obtained from the definition of $\valgen_n(\F)$ by replacing $\fhat_n$, which depends on the entire sequence $Z_{1:n}$, by $\fhat_{t-1}$, which depends only on $Z_{1:t-1}$, in the loss term that involves $P_t$. It can also be thought as a generalization of $\valonline$ because of the following. When the distribution $\pp$ degenerates to a point mass at a specific sequence $z_{1:n}$ then $P_t$ becomes a point mass at $z_t$ and the difference of cumulative losses above reduces to the individual sequence regret of $\fhat_{0:n-1}$ on $z_{1:n}$. This observation immediately gives us the following result.

\begin{lemma}
\label{lem:preqvsonline}
Fix any loss function $\ell$ and function class $\F$. Then we have $\valpreq_n(\F) \ge \valonline_n(\F)$.
\end{lemma}

The lemma above says that \newedit{prequential learnability is at least as hard as online learnability}. Our next lemma provides a converse result.

\begin{lemma}
\label{lem:preqvsonline2}
Fix any loss function and function class $\F$. Then for any sequence $\fhat_{0:n-1}$ of learning rules we have
\[
\valpreq_n(\fhat_{0:n-1}, \F) 
\le
\valonline_n(\fhat_{0:n-1}, \F)
+
2 \seqrad_n(\ell \circ \F) .
\]
This also means that
\[
\valpreq_n(\F) 
\le
\valonline_n(\F)
+
2 \seqrad_n(\ell \circ \F) .
\]
\end{lemma}
\begin{proof}
Let $\pp$ be an arbitrary distribution. We have the following three term decomposition:
\begin{align*}
& \quad \frac{1}{n} \sum_{t=1}^n \ell(P_t, \fhat_{t-1})
- \inf_{f \in \F} \frac{1}{n} \sum_{t=1}^n \ell(P_t, f) \\
&= \underset{(I)}{\underbrace{
\frac{1}{n} \sum_{t=1}^n \ell(P_t, \fhat_{t-1})
- \frac{1}{n} \sum_{t=1}^n \ell(Z_t, \fhat_{t-1})
}} \\
&\quad + \underset{(II)}{\underbrace{
\frac{1}{n} \sum_{t=1}^n \ell(Z_t, \fhat_{t-1})
- \inf_{f \in \F} \frac{1}{n} \sum_{t=1}^n \ell(Z_t, f)
}} \\
&\quad + \underset{(III)}{\underbrace{
\inf_{f \in \F} \frac{1}{n} \sum_{t=1}^n \ell(Z_t, f)
- \inf_{f \in \F} \frac{1}{n} \sum_{t=1}^n \ell(P_t, f)
}}.
\end{align*}
The term $(I)$ involves a martingale difference sequence $\ell(P_t, \fhat_{t-1}) - \ell(Z_t, \fhat_{t-1})$ and hence has expectation zero under $\pp$. Term $(II)$ is the individual sequence regret of $\fhat_{0:n-1}$ on the sequence $Z_{1:n}$ and hence is bounded, in expectation, by $\valonline_n(\fhat_{0:n-1}, \F)$.
Term $(III)$, in expectation, is at most,
\[
\E{
\sup_{f \in \F} 
\frac{1}{n} \sum_{t=1}^n \left(
\ell(Z_t, f) -  \ell(P_t, f)
\right)
}
\le 2 \seqrad_n(\ell \circ \F) ,
\]
where the inequality again follows from Theorem 2 of \citet{rakhlin2015sequential}.

The lemma now follows by taking expectations on both sides of the three term decomposition above and plugging in the upper bounds for each term's expected value.
\end{proof}

We now have all the ingredients to characterize \newedit{prequential learnability} for binary classification.

\begin{theorem}\label{thm:preq}
Consider binary classification with $0$-$1$ loss. Then all of the conditions in Theorem~\ref{thm:OLT} are also equivalent to:
\begin{itemize}
    \item $\F$ is \newedit{prequentially learnable.}
\end{itemize}
\end{theorem}
\begin{proof}
From Lemma~\ref{lem:preqvsonline}, we know that if a class is \newedit{prequentially learnable} then it is online learnable. In the other direction, using Lemma~\ref{lem:preqvsonline2}, we have
\begin{align*}
\valpreq_n(\F) 
&\le
\valonline_n(\F)
+
2 \seqrad_n(\ell \circ \F) \\
&\le
\valonline_n(\F)
+
\seqrad_n(\F) ,
\end{align*}
where the second inequality follows from Theorem~\ref{thm:radzeroone} in Appendix~\ref{sec:losstofunction}.
Under any of the equivalent conditions in Theorem~\ref{thm:OLT}, the $\lim\sup$ of both of the quantities on the right goes to zero as $n$ tends to infinity.
\end{proof}

\newedit{A similar result for the regression setting can be found in Section~\ref{sec:reg_prequential}.}

\section{The Regression Setting}
\label{sec:regression}

In this section, we provide analogues of most of the binary classification results for the regression setting with absolute loss. Note that rates of convergence can depend on the loss function but \emph{learnability} is quite robust to changes in the loss function. For example, we can also use squared loss $\ell((x,y), f) = (y - f(x))^2$. But we will keep our focus on the absolute loss in this section.

Our organization in this section is similar to the organization of results for binary classification. Section~\ref{sec:reg_statistical} and Section~\ref{sec:reg_online} review known results in iid and online learning, but give them a unified presentation. Section~\ref{sec:reg_general} and Section~\ref{sec:reg_prequential} present new results.

\subsection{Statistical Learning}
\label{sec:reg_statistical}

The fat shattering dimension of $\F$ is a scale-sensitive parameter that takes a scale $\gamma > 0$ as an argument. The fat shattering dimension of $\F$ at scale $\gamma$, denoted by $\fatgamma(\F)$, is the length $n$ of the longest sequence $x_{1:n}$ that is $\gamma$-shattered by $\F$. We say that a sequence $x_{1:n}$ is $\gamma$-shattered by $\F$ if there exists a witness sequence $s_{1:n}$ of real numbers such that
\[
\forall \epsilon_{1:n} \in \{\pm1\}^n, \exists f \in \F, \text{ s.t. } \forall t \in [n], \epsilon_t ( f(x_t) - s_t ) \ge \gamma .
\]

\begin{theorem}\label{thm:reg_SLT}
Consider regression with absolute loss in the iid statistical setting. Then, the following are equivalent:
\begin{enumerate}
    \item $\F$ is learnable.
    \item $\F$ is learnable via ERM.
    \item The ULLN condition~\eqref{eq:ULLN} holds for $\F$.
    \item $\forall \gamma > 0, \fatgamma(\F) < \infty$. \label{fatcond}
    \item $\limsup_{n \to \infty} \rad_n(\F) = 0$.
\end{enumerate}
\end{theorem}

The first four conditions are proved to be equivalent by~\citet{alon1997scale}. For connections between fat shattering dimension and Rademacher complexity see the work of~\citet{mendelson2002rademacher}.

\subsection{Online Setting}
\label{sec:reg_online}

The fat shattering dimension of $\F$ is replaced by its sequential analogue, just as VC dimension gets replaced by Littlestone dimension in the case of binary classification. The sequential fat shattering dimension of $\F$ at scale $\gamma$, denoted by $\sfatgamma(\F)$, is the depth $n$ of the deepest tree $\x$ that is $\gamma$-shattered by $\F$. We say that a complete binary tree $\x$ is $\gamma$-shattered by $\F$ if there exists a complete binary real valued witness tree $\s$ such that
\begin{equation*}
\forall \epsilon_{1:n} \in \{\pm1\}^n, \exists f \in \F, \text{ s.t. } \forall t \in [n], 
\epsilon_t( f(\x_t(\epsilon_{1:t-1}) - \s_t(\epsilon_{1:t-1}) ) \ge \gamma .
\end{equation*}

\begin{theorem}\label{thm:reg_OLT}
Consider regression with absolute loss in the online (individual sequence) setting. Then, the following are equivalent:
\begin{enumerate}
    \item $\F$ is learnable.
    \item The UMLLN condition~\eqref{eq:UMLLN} holds for $\F$.
    \item $\forall \gamma > 0, \sfatgamma(\F) < \infty$. \label{sfatcond}
    \item $\limsup_{n \to \infty} \seqrad_n(\F) = 0$. \label{regseqradcond}
\end{enumerate}
\end{theorem}

The last three conditions are shown to be equivalent in~\cite{rakhlin2015sequential} and the connection with learnability was established in~\cite{rakhlin2015online}.

As in the binary classification setting, online learnability is harder than iid statistical learnability. That is, for any $\F$ and any $\gamma > 0$, $\fatgamma(\F) \le \sfatgamma(\F)$ and the gap in this inequality can be arbitrarily large. For example, the set $\F_{\text{bv}}$ of bounded variation functions from $[0,1]$ to $[0,1]$ with total variation at most $V$, 
has $\fatgamma(\F_{\text{bv}}) < 1 + V/\gamma$ for all $\gamma > 0$
but $\sfatgamma(\F_{\text{bv}}) = \infty$ for all $\gamma > 0$.

\subsection{Learnability under General Stochastic Processes}
\label{sec:reg_general}

We first state an analogue of Theorem~\ref{thm:ldimlowerbound} for the regression setting.

\begin{theorem}
\label{thm:sfatlowerbound}
Consider regression with absolute loss in the general stochastic process setting. Suppose the class $\F$ is not online learnable, i.e., there exists $\gamma > 0$ such that $\sfatgamma(\F) = \infty$. Then for any $n \ge 1$,
$\valgen_n(\F) \ge \gamma/500$. Therefore, the class $\F$ is not \newedit{process learnable.}
\end{theorem}

\color{black}
The result above allows us to extend Theorem~\ref{thm:general} to the regression setting.

\begin{theorem}\label{thm:reg_general}
Consider regression with absolute loss with a class $\F$. Then all of the equivalent conditions in Theorem~\ref{thm:reg_OLT} are also equivalent to:
\begin{itemize}
    \item $\F$ is process learnable.
\end{itemize}
\end{theorem}
\begin{proof}
Theorem~\ref{thm:sfatlowerbound} shows that process learnability implies online learnability. For the other direction, note that according to Theorem~\ref{thm:gentorad} we have
\[
\valgen_n(\F) \le 4 \, \seqrad_n(\ell \circ \F) .
\]
Taking $\lim \sup$ of both sides as $n$ tends to infinity and using Theorem~\ref{thm:radabsolute} in Appendix~\ref{sec:losstofunction} to bound the right hand side gives the desired implication.
\end{proof}

\subsection{Learnability under the Prequential Version}
\label{sec:reg_prequential}

\begin{theorem}\label{thm:reg_preq}
Consider regression with absolute loss with a class $\F$. Then all of the equivalent condition in Theorem~\ref{thm:reg_OLT} are also equivalent to:
\begin{itemize}
    \item $\F$ is prequentially learnable.
\end{itemize}
\end{theorem}
\begin{proof}
From Lemma~\ref{lem:preqvsonline}, we know that if a class is prequentially learnable then it is online learnable. For the other direction, using Lemma~\ref{lem:preqvsonline2}, we have
\begin{align*}
\valpreq_n(\F) 
&\le
\valonline_n(\F)
+
2 \seqrad_n(\ell \circ \F) \ .
\end{align*}
From online learnability of $\F$ and Theorem~\ref{thm:radabsolute} in Appendix~\ref{sec:losstofunction}, we know that the $\lim\sup$ of both the quantities on the right is at most zero as $n$ tends to infinity, giving the desired result.
\end{proof}
\color{black}

\section{Conclusion}

In this paper we have proposed two new definitions of learnability of a class of functions under general non-iid stochastic processes. For the first definition, we showed that learnability is equivalent to online learnability. This equivalence also holds for the second definition, which is a prequential version of the first. We also showed how to extend our results from binary classification to the regression setting.

Our work poses several interesting questions for further investigation.  First, we defined learnability using expectations. It will be good to derive high probability results. Second, we ignored the issue of convergence rates for simplicity. It should be possible to extend our analysis to extract information about rates of convergence. This is because the tools from iid and online learning that we use are powerful enough to give us information about rates. Third, instead of using a normalizing factor of $n$, the sample size, other data dependent normalizing factors could be of interest in applications. In this context, the theory of self-normalized processes comes to mind \citep{pena2008self}. Fourth, iid learning theory has been extended to deal with privacy constraints. Starting from the seminal work of \citet{kasiviswanathan2011can} these efforts have looked at a formalization of user privacy known as {\em differential privacy}. Surprisingly, recent work \citep{alon2019private,bun2020equivalence} has shown that iid learnability under the additional constraint of approximate differential privacy is equivalent to online learnability! It will be interesting to study learnability under general stochastic processes with additional privacy constraints on the learning algorithm. Last, but certainly not least, there is a need to connect various strands of learning theory research on non-iid processes. It is unlikely that there is a single definition of learnability that is satisfactory for all purposes. We hope we have proposed two interesting and useful ones. Comparing and contrasting various existing definitions and approaches is an important goal for future work in this area.

%% file: appendix_content.tex
\section{Relating the complexity of the loss class to the function class}
\label{sec:losstofunction}

\newedit{We first consider the $0$-$1$ loss function and then give a result for the absolute loss.}

\subsection{Zero-One Loss}

The result below is essentially already known. The main ideas are present in published work~\citep{rakhlin2011online}. We just present the result in a form that is immediately useful to us.
But before we do that, we need a useful lemma.

\begin{lemma}
\label{lem:stillrad}
For any sequence $\epsilon_{1:n}$ of iid Rademacher random variables and any fixed $\{\pm1\}$-valued tree $\s$, the sequence $\left(\epsilon_t \s_t(\epsilon_{1:t-1}) \right)_{t=1}^n$ is also iid Rademacher.
\end{lemma}
\begin{proof}
It is easy to see that the sequence of random variables we have constructed takes values in $\{\pm1\}$. We just need to check that the distribution of $\epsilon_t \s_t(\epsilon_{1:t-1})$ conditioned on $\epsilon_{1:t-1}$ is a fixed distribution independent of the past. This is readily verified since
\begin{align*}
\E{ \epsilon_t \s_t(\epsilon_{1:t-1}) | \epsilon_{1:t-1} } = \s_t(\epsilon_{1:t-1}) \E{ \epsilon_t | \epsilon_{1:t-1} } 
= 0 .
\end{align*}
Therefore, we have shown that the distribution of $\epsilon_t \s_t(\epsilon_{1:t-1)} $ conditioned on $\epsilon_{1:t-1}$ is always Rademacher (symmetric Bernoulli).
\end{proof}

Now we are ready to state and prove the main result of this subsection.

\begin{theorem}
\label{thm:radzeroone}
Let $\F$ be a binary valued function class and let $\ell$ be the $0$-$1$ loss function. Then we have,
\[
\seqrad_n(\ell \circ \F)
= \frac{1}{2} \seqrad_n(\F) .
\]
\end{theorem}
\begin{proof}
Instead of using $\Z$-valued trees to define $\seqrad_n(\ell \circ \F)$, we will use a pair $\x, \y$ of $\X$- and $\Y$-valued trees. The equality we are trying to prove can then be written as:
\begin{equation*}
2\,\sup_{\x,\y} \E{
\sup_{f \in \F} \frac{1}{n} \sum_{t=1}^n \epsilon_t
\ind{ \y_t(\epsilon_{1:t-1}) \neq
f(\x_t(\epsilon_{1:t-1})) }
} 
= \sup_{\x} \E{
\sup_{f \in \F} \frac{1}{n} \sum_{t=1}^n \epsilon_t
f(\x_t(\epsilon_{1:t-1})) 
}.
\end{equation*}
For $y_1,y_2 \in \{\pm 1\}$, we can write $2 \ind{y_1 \neq y_2}$ as $1-y_1y_2$. Note that sequential Rademacher complexity is not affected if the entire function classes is shifted by a constant. Therefore, the left hand side is equal to
\[
\sup_{\x,\y} \E{
\sup_{f \in \F} \frac{1}{n} \sum_{t=1}^n -\epsilon_t
\y_t(\epsilon_{1:t-1})
f(\x_t(\epsilon_{1:t-1})) 
}.
\]
Now consider the $\{ \pm1 \}$-valued tree $\s = -\y$. From
Lemma~\ref{lem:stillrad}, we know the above is equal to
\[
\sup_{\x,\s} \E{
\sup_{f \in \F} \frac{1}{n} \sum_{t=1}^n \epsilon_t
f(\x_t(\epsilon_{1:t-1} \cdot \s_{1:t-1}(\epsilon)))
} ,
\]
where $\epsilon_{1:t-1} \cdot \s_{1:t-1}(\epsilon)$ denotes the sequence
\[
\epsilon_1 \cdot \s_1, \epsilon_2 \cdot \s_2(\epsilon_1), \ldots, \epsilon_{t-1} \cdot \s_{t-1}(\epsilon_{1:t-2}).
\]
Define the tree $\x'$ as $\x'(\epsilon_{1:t}) = \x_t(\epsilon_{1:t-1} \cdot \s_{1:t-1}(\epsilon))$
and note that as $\x$ ranges over all $\X$-valued trees and $\s$ ranges over all $\{\pm 1\}$-valued trees, $\x'$ ranges over all $\X$-valued trees. Therefore, the supremum over the pair $\x,\s$ above can simply be written as
\[
\sup_{\x'} \E{
\sup_{f \in \F} \frac{1}{n} \sum_{t=1}^n \epsilon_t
f(\x'_t(\epsilon_{1:t-1}))
}. 
\]
As we noted, $\x'$ ranges over all $\X$-valued trees making the above quantity the same as $\seqrad_n(\F)$ which finishes the proof. 
\end{proof}

\subsection{Absolute Loss}

\color{black}

We now consider the absolute loss case. The only property of the absolute loss used in the proof below is that it is 1-Lipschitz (in either argument provided the other one is fixed).

\begin{theorem}
\label{thm:radabsolute}
Let $\F$ be a bounded real valued function class such that $\sfatgamma(\F) < \infty$ for all $\gamma > 0$. Let $\ell$ be the absolute loss. Then we have
\[
\limsup_{n \to \infty} \seqrad_n(\ell \circ \F) \le 0 \ .
\]
\end{theorem}
\begin{proof}
In this proof $c, C$ will denote universal constants that can change from line to line.
From Corollary 10 of \citet{block2021majorizing}, we have
\begin{equation}
\seqrad_n(\ell \circ \F) \le C \cdot
\inf_{\alpha \ge 0}
\left(
\alpha + \frac{1}{\sqrt{n}} \int_{\alpha}^1 \sqrt{ \log N'(\ell \circ \F, \delta)} d \delta 
\right) \ .
\label{eq:radtofrac}
\end{equation}
Here $N'(\ell \circ \F, \delta)$ refers to the fractional covering number of the class $\ell \circ \F$ as defined in Definition 6 of \citet{block2021majorizing}. From the definition of the fractional covering number and the fact that the absolute loss is $1$-Lipschitz, we have $ N'(\ell \circ \F, \delta) \le N'(\F, \delta)$. Moreover, by Theorem 13 of \citet{block2021majorizing}, we have 
\[
 N'(\F, \delta) \le \left(\frac{C}{\delta}\right)^{3 \, \sfat_{c \delta}(\F)} \ .
\]
Fix some $\alpha > 0$. Plugging the bound above into~\eqref{eq:radtofrac} gives us,
\[
\seqrad_n(\ell \circ \F) \le C \cdot
\left(
\alpha + \frac{1}{\sqrt{n}} \int_{\alpha}^1 \sqrt{ \sfat_{c \delta}(\F) \cdot \log(1/\delta) } d \delta 
\right) \ .
\]
Now since $\sfat_{c \delta}(\F) < \infty$ for all $\delta > 0$, taking limits with respect to $n$ on both sides gives us
\[
\limsup_{n \to \infty} \seqrad_n(\ell \circ \F) \le C \cdot \alpha \ .
\]
Since $\alpha > 0$ was arbitrary this proves the result.
\end{proof}
\color{black}

\section{Proofs of Lower Bounds}

\begin{proof}[Proof of Theorem~\ref{thm:ldimlowerbound}]
Since the Littlestone dimension of $\F$ is infinite, by Theorem 3 of \citet{alon2019private} it contains $N = 2^{2^n+1}$ thresholds. This means that there are $N$ functions $f_1,\ldots,f_{N} \in \F$ and $N$ examples
$x_1,\ldots,x_{N} \in \X$ such that for all $i,j\le N$,
\[
f_j(x_i) = 1 \text{ if and only if } i \le j .
\]
Without loss of generality identify these $m$ examples with integers $1$ through $N$ written in binary notation (with enough zero padding to the left to make the binary encoding a bit vector of length exactly $2^n+1$) and the functions with threshold functions
\[
x \mapsto \ind{x \le \bitv}
\]
for bit vectors $\bitv$ of length $2^n+1$.

We will now define a stochastic process $(X_t,Y_t), 1 \le t \le n$ indexed by bit vectors $\bitv$ of length $2^n$. The labels $Y_t$ will be deterministic given $X_t$ chosen as $Y_t = \ind{X_t \le \bitv 1}$ where $\bitv 1$ has a $1$ added at the end and is therefore of length $2^n+1$. So we only need to define a process $X_1,\ldots,X_n$. This is defined as follows. Let $\epsilon_{1:n}$ be iid Rademacher random variables.

\begin{itemize}
    \item $\ell_0 = 0$
    \item For $t = 1$ to $n$
    \begin{itemize}
        \item If $\epsilon_t = +1$: $\ell_{t} = \ell_{t-1} + 2^{n-t}$ else: $\ell_{t} = \ell_{t-1}$ 
        \item $\bitv_t = (2^n+1)$-length bit vector with same $\ell_t$-length prefix as $\bitv$, $(\ell_t+1)$st bit equal to 1, and rest padded with zeros
        \item Output $X_t = \bitv_t$
    \end{itemize}
\end{itemize}

If all $\epsilon_t$'s turn out to be $+1$ (an event with probability $2^{-n}$), $\ell_n$ can become as large as
\[
2^{n-1} + 2^{n-2} + \ldots + 2 + 1 = 2^n - 1
\]
which still leaves two bits to add a $``10"$ at the end. So we do have enough bits available for all possible increases of resolution. Note that the true function is a threshold at an odd integer whereas the sampled $X_t$'s are always even integers. Finally, note that by construction $Y_t = \ind{X_t \le \bitv 1} = \ind{\bitv_t \le \bitv 1} = \bitv[\ell_t+1]$ where $\bitv[\ell]$ is the $\ell$th bit of $\bitv$.

Denote the stochastic process defined above by $\pp_\bitv$. Our proof will follow the probabilistic method replacing the supremum over $\pp$ in the definition of $\valgen_n(\F)$ with an expectation over $\pp$ with $\pp$ chosen to be $\pp_\bitv$ with the bit vector $\bitv$ chosen uniformly at random. That is,
\begin{align*}
\valgen_n(\fhat_n, \F)
&= \sup_{\pp} \E{
 \frac{1}{n} \sum_{t=1}^n \ell(P_t, \fhat_n)
                        - \inf_{f \in \F} \frac{1}{n} \sum_{t=1}^n \ell(P_t, f)
} \\
&\ge \Es{\bitv}{\Es{\pp_\bitv}{
 \frac{1}{n} \sum_{t=1}^n \ell(P_t, \fhat_n)
                        - \inf_{f \in \F} \frac{1}{n} \sum_{t=1}^n \ell(P_t, f)
}} .
\end{align*}
Note that for every $\bitv$, $Y_t = f_\bitv(X_t)$ for some $f_\bitv \in \F$ which means that the infimum above is zero for every $\bitv$. Therefore, we have
\begin{align*}
\valgen_n(\fhat_n, \F)
&\ge \Es{\bitv}{\Es{\pp_\bitv}{
 \frac{1}{n} \sum_{t=1}^n \ell(P_t, \fhat_n)
}} .
\end{align*}
Now we want to argue that
$\forall t \in \{1,\ldots,n\}$, 
\begin{align}
\label{eq:lowerbdtoshow}
\E{
\ell(P_t, \fhat_n)
} \ge \frac{1}{8} .
\end{align}
Note the two sources of randomness in this expectation: the random bit vector $\bitv$ and the Rademacher random variables $\epsilon_{1:n}$. Also, note that $\fhat_n$ does not have access to $\bitv$ but only to the observed sample
\[
(X_1,Y_1),
\ldots,
(X_n, Y_n)
=
(\bitv_1,\bitv[\ell_1+1]),
\ldots,
(\bitv_n,\bitv[\ell_n+1])
\]
Note that $P_t$ puts probability $1/2$ each on a ``high resolution'' value $\bitv_t^+$ (corresponding to $\epsilon_t=+1$) and a ``low resolution'' value $\bitv_t^-$ (corresponding to $\epsilon_t = -1$). Therefore, we have,
\[
\ell(P_t, \fhat_n)
= \frac{1}{2} \ind{ \fhat_n(\bitv_t^-) \neq f_\bitv(\bitv_t^-) }
+ \frac{1}{2} \ind{ \fhat_n(\bitv_t^+) \neq f_\bitv(\bitv_t^+) } .
\]
The expectation of the quantity above can be lower bounded as,
\begin{align}
\notag
\E{
\ell(P_t, \fhat_n)
}
&\ge
\frac{1}{2} \E{
 \ind{ \fhat_n(\bitv_t^+) \neq f_\bitv(\bitv_t^+) }
} \\
\notag
&\ge \frac{1}{2}
\prob{\epsilon_t = -1}
\E{
 \ind{ \fhat_n(\bitv_t^+) \neq f_\bitv(\bitv_t^+) \middle| \epsilon_t = -1 }
} \\
&=
\label{eq:condepsilon}
\frac{1}{4}
\E{
 \ind{ \fhat_n(\bitv_t^+) \neq f_\bitv(\bitv_t^+) \middle| \epsilon_t = -1 }
} .
\end{align}
Now note that $f_\bitv(\bitv_t^+)$ is simply equal to $\bitv[\ell_{t-1}+2^{n-t}+1]$. Further note that when $\epsilon_t = -1$, the largest that $\ell_n$ can become is
\[
\ell_{t-1} + 2^{n-t-1} + 2^{n-t-2} + \ldots + 1 = \ell_{t-1} + 2^{n-t} - 1 \ .
\]
This means that, conditioned on $\epsilon_t = -1$, the entire sample is measurable w.r.t. $\bitv[1:\ell_{t-1}+2^{n-t}]$. Since $\bitv_t^+$ is measurable w.r.t. $\bitv[1:\ell_{t-1}+2^{n-t}]$, we can conclude that $f_\bitv(\bitv_t^+)$ is independent of the sample and $\bitv_t^+$ conditioned on $\epsilon_t=-1$. Since the unconditional distribution of $\bitv[\ell_{t-1}+2^{n-t}+1]$ is uniform on $\{0,1\}$, this implies that
\begin{equation}
\label{eq:condepsilonlowerbd}
\E{
 \ind{ \fhat_n(\bitv_t^+) \neq f_\bitv(\bitv_t^+) \middle| \epsilon_t = -1 }
} \ge \frac{1}{2}
\end{equation}
which along with~\eqref{eq:condepsilon} gives~\eqref{eq:lowerbdtoshow}.
\end{proof}

\begin{proof}[Proof of Theorem~\ref{thm:sfatlowerbound}]
Since $\sfatgamma(\F) = \infty$, by Theorem 8 of \citet{jung2020equivalence}, $\F$ contains $N=2^{2^n+1}$ thresholds with margin $\gamma/5$. This means that there are $N$ functions $\tf_1,\ldots,\tf_N \in \F$, $N$ examples $x_1,\ldots,x_N \in \X$ and $u,u'
 \in [-1,+1]$ such that $|u-u'|\ge \gamma/5$, $|\tf_j(x_i)-u|\le \gamma/100$ if $i \le j$, and $|\tf_j(x_i)-u'| \le \gamma/100$ if $i > j$.
 
Without loss of generality assume that $u > u'$. Suppose $\F$ is learnable to accuracy $\epsilon$. We want to show that by choosing $\epsilon$ to be sufficiently small, e.g., $\epsilon = \gamma/500$,  we can get a contradiction with the lower bound established in proof of Theorem~\ref{thm:ldimlowerbound} above. Consider the stochastic process $X_t,Y_t$ constructed in the proof above. We transform it into the regression setting by converting the binary labels $Y_t$ into real values $\tilde{Y}_t$ as follows:
\begin{equation*}
\tilde{Y}_t = 
\begin{cases}
u & \text{if } Y_t = 1\ ,\\
u' & \text{if } Y_t = 0\ .
\end{cases}
\end{equation*}
The transformed data consisting of pairs $(X_t,\tilde{Y}_t)$ is fed into the $\epsilon$-accurate learner $\fhat_n$ for which we have the guarantee that
\begin{equation*}
    \E{
 \frac{1}{n} \sum_{t=1}^n \ell(\tP_t, \fhat_n)
                        - \inf_{f \in \F} \frac{1}{n} \sum_{t=1}^n \ell(\tP_t, f)
} \le \epsilon
\end{equation*}
where $\ell$ is the absolute loss and $\tP_t$ is the conditional distribution defined with respect to the transformed process $X_t,\tilde{Y}_t$. Note that the labels $Y_t$ were generated using a binary threshold function $f_j$ which can be approximated to an error within  $\gamma/100$ (in supremum norm) by some $\tf_j$. This means that
\[
\inf_{f \in \F} \frac{1}{n} \sum_{t=1}^n \ell(\tP_t, f)
\le \gamma/100 .
\]
Therefore, for some $t$,
\[
\E{
 \ell(\tP_t, \fhat_n)
 } \le \epsilon + \gamma/100
\]
Now consider the \emph{binary classifier} $\fhat'_n$ obtained from the real valued function $\fhat_n$ as follows:
\begin{equation*}
\fhat'_n(x) = 
\begin{cases}
1 & \text{if } \fhat_n(x) > (u+u')/2\ ,\\
0 & \text{otherwise}\ .
\end{cases}
\end{equation*}
We know from~\eqref{eq:condepsilon} and~\eqref{eq:condepsilonlowerbd} that $\fhat'_n$ will predict the label of $\bitv_t^+$ incorrectly with probability at least $1/4$. Because of the $\gamma/5$ gap between $u$ and $u'$ this means that $\fhat_n$ incurs an absolute loss of at least $\gamma/10$ on $\bitv_t^+$ and its real valued label (which is either $u$ or $u'$) with probability at least $1/4$. Recalling that $\tP_t$ puts probability mass $1/2$ on $\bitv_t^+$, we therefore have
\[
\E{
 \ell(\tP_t, \fhat_n)
 } \ge \frac{1}{2} \cdot \frac{1}{4} \cdot \frac{\gamma}{10} = \frac{\gamma}{80}
\]
which means
\[
\frac{\gamma}{80} \le \epsilon + \frac{\gamma}{100} .
\]
In order for this to give us a contradiction we just need to ensure that $\epsilon < \gamma/400$. The choice $\epsilon = \gamma/500$ does that.
\end{proof}